\documentclass[sigconf, screen, dvipsnames, anonymous=false, review=false]{acmart}
\acmSubmissionID{92}

\usepackage[noend]{algorithmic}
\usepackage{algorithm}
\usepackage{multirow}
\usepackage{amsmath}
\usepackage{bbm}
\usepackage[inline]{enumitem} 
\usepackage{subcaption}
\usepackage{xcolor}

\usepackage{hyperref}
\usepackage{lipsum}
\usepackage{tabularx}
\usepackage{float}
\usepackage{multicol, blindtext}
\usepackage{numprint}

\usepackage{wrapfig}

\usepackage{mleftright}

\usepackage{arydshln}
\makeatletter
\def\adl@drawiv#1#2#3{
        \hskip.5\tabcolsep
        \xleaders#3{#2.5\@tempdimb #1{1}#2.5\@tempdimb}%
                #2\z@ plus1fil minus1fil\relax
        \hskip.5\tabcolsep}
\newcommand{\cdashlinelr}[1]{%
  \noalign{\vskip\aboverulesep
          \global\let\@dashdrawstore\adl@draw
          \global\let\ adl@draw\adl@drawiv}
  \cdashline{#1}
  \noalign{\global\let\adl@draw\@dashdrawstore
          \vskip\belowrulesep}}
\makeatother

\usepackage{tikz}
\usetikzlibrary{automata, arrows, bayesnet, bending}

\usepackage{tikzscale}

\DeclareMathOperator*{\argmax}{arg\,max}
\DeclareMathOperator*{\argmin}{arg\,min}

\newtheorem{theorem}{Theorem}

\allowdisplaybreaks

\copyrightyear{2024}
\acmYear{2024}
\setcopyright{rightsretained}
\acmConference[RecSys '24]{18th ACM Conference on Recommender
Systems}{October 14--18, 2024}{Bari, Italy}
\acmBooktitle{18th ACM Conference on Recommender Systems (RecSys '24),
October 14--18, 2024, Bari, Italy}\acmDOI{10.1145/3640457.3688105}
\acmISBN{979-8-4007-0505-2/24/10}

\usepackage{etoolbox}
\makeatletter
\patchcmd{\maketitle}{\@copyrightpermission}{
  \begin{minipage}{0.3\columnwidth}
    \href{http://creativecommons.org/licenses/by/4.0/}{\includegraphics[width=0.90\textwidth]{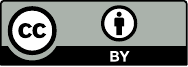}}
  \end{minipage}\hfill
  \begin{minipage}{0.7\columnwidth}
    \href{http://creativecommons.org/licenses/by/4.0/}{This work is licensed under a Creative Commons Attribution International 4.0 License.}
  \end{minipage}
  \vspace{5pt}
}{}{}
\makeatother

\begin{document}
\title{Optimal Baseline Corrections for Off-Policy Contextual Bandits}

\author{Shashank Gupta}
\affiliation{
  \institution{University of Amsterdam}
  \city{Amsterdam}
  \country{The Netherlands}
} 
\authornote{Equal contribution.}
\email{s.gupta2@uva.nl}

\author{Olivier Jeunen}
\affiliation{
  \institution{ShareChat}
  \city{Edinburgh}
  \country{United Kingdom}
}
\authornotemark[1]
\email{jeunen@sharechat.co}

\author{Harrie Oosterhuis}
\affiliation{
  \institution{Radboud University}
  \city{Nijmegen}
  \country{The Netherlands}
} 
\email{harrie.oosterhuis@ru.nl}

\author{Maarten de Rijke}
\affiliation{
  \institution{University of Amsterdam}
  \city{Amsterdam}
  \country{The Netherlands}
} 
\email{m.derijke@uva.nl}

\begin{abstract}
The off-policy learning paradigm allows for recommender systems and general ranking applications to be framed as \emph{decision-making} problems, where we aim to learn decision policies that optimize an unbiased \emph{offline} estimate of an \emph{online} reward metric.
With unbiasedness comes potentially high variance, and prevalent methods exist to reduce estimation variance.
These methods typically make use of control variates, either additive (i.e., baseline corrections or doubly robust methods) or multiplicative (i.e., self-normalisation). 

Our work unifies these approaches by proposing a single framework built on their equivalence in learning scenarios.
The foundation of our framework is the derivation of an equivalent baseline correction for all of the existing control variates.
Consequently, our framework enables us to characterize the variance-optimal unbiased estimator and provide a closed-form solution for it.
This optimal estimator brings significantly improved performance in both evaluation and learning, and minimizes data requirements.
Empirical observations corroborate our theoretical findings.
\end{abstract}

\ccsdesc[500]{Information systems~Recommender systems}

\keywords{Contextual Bandits; Recommender Systems; Off-policy Learning}

\maketitle


\section{Introduction \& Motivation}
Recommender systems have undergone a paradigm shift in the last few decades, moving their focus from \emph{rating} prediction in the days of the Netflix Prize~\cite{Bennet2007}, to \emph{item} prediction from implicit feedback~\cite{Rendle2022} and \emph{ranking} applications gaining practical importance~\cite{Steck2013,Jeunen2023_nDCG}.
Recently, work that applies ideas from the algorithmic \emph{decision-making} literature to recommendation problems has become more prominent~\cite{Vasile2020,Saito2021,Gupta2024, jeunen2022consequences}.
While this line of research is not inherently new~\cite{Shani2002,Li2010}, methods based on contextual bandits (or reinforcement learning by extension) have now become widespread in the recommendation field~\cite{McInerney2018,Mehrotra2020,Bendada2020, Jeunen2021_TopK, Yi2023, Su2024, Briand2024}.
The \emph{off-policy} setting is particularly attractive for practitioners~\cite{vandenAkker2024}, as it allows models to be trained and evaluated in an offline manner~\cite{chen2019top, Dong2020, ma2020off, Jeunen2020, Jeunen2021_Pessimism, Chen2021, Liu2022, chen2022actorcritic, Jeunen2023,gupta2023safe,gupta2023deep,gupta-2024-practical,gupta-2023-first-abstract}.
Indeed, methods exist to obtain unbiased \emph{offline} estimators of \emph{online} reward metrics, which can then be optimized directly~\cite{Jeunen2021_Thesis}.

Research at the forefront of this area typically aims to find Pareto-optimal solutions to the bias-variance trade-off that arises when choosing an estimator: reducing variance by accepting a small bias~\cite{Ionides2008,su2020doubly}, by introducing control variates~\cite{Dudik2014,Swaminathan2015}, or both~\cite{Su2019}.
Control variates are especially attractive as they (asymptotically) preserve the unbiasedness of the widespread inverse propensity scoring (IPS) estimator.
Additive control variates give rise to baseline corrections~\cite{Greensmith2004}, regression adjustments~\cite{Freedman2008}, and doubly robust estimators~\cite{Dudik2014}.
Multiplicative control variates lead to self-normalised estimators~\cite{Kong1992,Swaminathan2015}.
Previous work has proven that for off-policy \emph{learning} tasks, the multiplicative control variates can be re-framed using an equivalent additive variate~\cite{Joachims2018,Budylin2018}, enabling mini-batch optimization methods to be used.
We note that the self-normalised estimator is only \emph{asymptotically} unbiased: a clear disadvantage for evaluation with finite samples.
The common problem which most existing methods tackle is that of \emph{variance reduction} in offline value estimation, either for learning or for evaluation.
The common solution is the application of a control variate, either multiplicative or additive~\cite{Owen2013}.
However, to the best of our knowledge, there is no  work that attempts to unify these methods.
Our work addresses this gap by presenting these methods in a unifying framework of baseline corrections which, in turn, allows us to find the optimal baseline correction for variance reduction.

In the context of off-policy learning, adding to the well-known equivalence between reward-translation and self-normalisation described by \citet{Joachims2018}, we demonstrate that the equivalence extends to baseline corrections, regression adjustments, and doubly robust estimators with a constant reward model.
Further, we derive a novel baseline correction method for off-policy learning that minimizes the variance of the gradient of the (unbiased) estimator. We further show that the baseline correction can be estimated in a closed-form fashion, allowing for easy practical implementation.

In line with recent work on off-policy evaluation/learning for recommendation~\cite{Jeunen2021_TopK,Jeunen2023_AuctionGym,Jeunen2021_Pessimism,rohde2018recogym,Saito2021_OPE}, we adopt an off-policy simulation environment to emulate real-world recommendation scenarios, such as stochastic rewards, large action spaces, and controlled randomisation.
This choice also encourages future reproducibility~\cite{saito2020open}.
Our experimental results indicate that our proposed baseline correction for gradient variance reduction enables substantially faster convergence and lower gradient variance during learning.

In addition, we derive a closed-form solution to the optimal baseline correction for off-policy evaluation, i.e., the one that minimizes the variance of the estimator itself. 
Importantly, since our framework only considers unbiased estimators, the variance-optimality implies overall optimality.
Our experimental results show that this leads to lower errors in policy value estimation than widely used doubly-robust and SNIPS estimators~\cite{Dudik2014, Swaminathan2015}.

All source code to reproduce our experimental results is available at: \url{https://github.com/shashankg7/recsys2024_optimal_baseline}.





\section{Background and Related Work}\label{sec:sec1}
The goal of this section is to introduce common contextual bandit setups for recommendation, both on-policy and off-policy.

\subsection{On-policy contextual bandits}
\label{sec:general}

We address a general contextual bandit setup~\cite{saito2022counterfactual,joachims2016counterfactual} with contexts $X$, actions $A$, and rewards $R$.
The context typically describes \emph{user} features, actions are the \emph{items} to recommend, and rewards can be any type of \emph{interaction} logged by the platform.
A policy $\pi$ defines a conditional probability distribution over actions $x$: $\mathsf{P}(A=a\mid X=x,\Pi=\pi) \equiv \pi(a \mid x)$.
Its \emph{value} is the expected reward it yields:
\begin{equation}\label{eq:onpolicy_reward}
    V(\pi) = \mathop{\mathbb{E}}_{x\sim\mathsf{P}(X)}\Big[\mathop{\mathbb{E}}_{a \sim \pi(\cdot \mid x)}\mleft[ R \mright]\Big].
\end{equation}
When the policy $\pi$ is deployed, we can estimate this quantity by averaging the rewards we observe.
We denote the expected reward for action $a$ and context $r$ as $r(a,x) \coloneqq \mathbb{E}[R \mid X=x;A=a]$.

In the field of contextual bandits (and reinforcement learning (RL) by extension), one often wants to learn $\pi$ to maximise $V(\pi)$~\cite{sutton2018reinforcement,lattimore2020bandit}.
This is typically achieved through gradient ascent.
Assuming $\pi_{\theta}$ is parameterised by $\theta$, we iteratively update with learning rate $\eta$:
\begin{equation}
    \theta_{t+1} = \theta_{t} + \eta \nabla_{\theta}(V(\pi_{\theta})).
\end{equation}
Using the well-known REINFORCE ``log-trick''~\cite{Williams1992}, the above gradient can be formulated as an expectation over sampled actions, whereby tractable Monte Carlo estimation is made possible:
\begin{align}
    \nabla_{\theta}(V(\pi_{\theta})) &=  \nabla_{\theta}\mleft( \mathop{\mathbb{E}}_{x\sim\mathsf{P}(X)}\mleft[\mathop{\mathbb{E}}_{a \sim \pi_{\theta}(\cdot\mid x)}\mleft[ R \mright]\mright] \mright) \nonumber \\
    &= \nabla_{\theta}\mleft( \int\sum_{a \in \mathcal{A}} \pi_{\theta}(a\mid x) r(a,x) \mathsf{P}(X=x) \rm{d}x   \mright) \nonumber \\
    &=  \int\sum_{a \in \mathcal{A}} \nabla_{\theta}\mleft( \pi_{\theta}(a\mid x) r(a,x) \mright)\mathsf{P}(X=x)\rm{d}x \label{eq:mc_sample}   \\
    &=  \int\sum_{a \in \mathcal{A}} \pi_{\theta}(a|x) \nabla_{\theta}\mleft( \log(\pi_{\theta}(a|x)) r(a,x) \mright) \mathsf{P}(X=x)\rm{d}x \nonumber \\
    &=  \mathop{\mathbb{E}}_{x\sim\mathsf{P}(X)}\mleft[\mathop{\mathbb{E}}_{a \sim \pi_{\theta}(\cdot\mid x)}\mleft[ \nabla_{\theta}\mleft( \log(\pi_{\theta}(a\mid x)) R \mright) \mright]\mright]. \nonumber
\end{align}
This provides an unbiased estimate of the gradient of $V(\pi_{\theta})$. However, it may be subject to high variance due to the inherent variance of $R$.
Several techniques have been proposed in the literature that aim to alleviate this, mostly using additive \emph{control variates}.

Control variates are random variables with a known expectation~\cite[\S 8.9]{Owen2013}.
If the control variate is correlated with the original estimand --- in our case $V(\pi_{\theta})$ --- they can be used to reduce the estimator's variance.
A natural way to apply control variates to a sample average estimate for Eq.~\ref{eq:onpolicy_reward} is to estimate a model of the reward $\widehat{r}(a,x)\approx \mathbb{E}[R|X=x;A=a]$ and subtract it from the observed rewards~\cite{Freedman2008}.
This is at the heart of key RL techniques (i.a., generalised advantage estimation~\cite{Schulmanetal_ICLR2016}), and it underpins widely used methods to increase sensitivity in online controlled experiments~\cite{Deng2013,Poyarkov2016,Budylin2018,Baweja2024}.
As such, it applies to both \emph{evaluation} and \emph{learning} tasks.
We note that if the model $\widehat{r}(a,x)$ is biased, this bias propagates to the resulting estimator for $V(\pi_{\theta})$.

Alternatively, instead of focusing on reducing the variance of $ V(\pi_{\theta})$ directly, other often-used approaches tackle the variance of its gradient estimates $\nabla_{\theta}(V(\pi_{\theta}))$ instead.

Observe that $\mathop{\mathbb{E}}_{a \sim \pi_{\theta}(\cdot|x)}\mleft[ \nabla_{\theta}\mleft( \log(\pi_{\theta}(a\mid x))\mright) \mright]=0$~\cite[Eq. 12]{Mohamed2020}.
This implies that a translation on the rewards in Eq.~\ref{eq:mc_sample} does not affect the unbiasedness of the gradient estimate.
Nevertheless, as such a translation can be framed as an additive control variate, it will affect its variance.
Indeed, ``\emph{baseline corrections}'' are a well-known variance reduction method for on-policy RL methods~\cite{Greensmith2004}.
For a dataset consisting of logged contexts, actions and rewards $\mathcal{D} = \{(x_i,a_i,r_i)_{i=1}^{N}\}$, we apply a \emph{baseline} control variate $\beta$ to the estimate of the final gradient to obtain:
\begin{equation}
\begin{split}
    \nabla_{\theta}(V(\pi_{\theta}))
    &\approx
    \widehat{\nabla_{\theta}(V_{\beta}(\pi_{\theta}))} 
    \\
    &=
    \frac{1}{\mleft|\mathcal{D}\mright|}
    \sum_{(x,a,r) \in \mathcal{D}}   (r-\beta) \nabla_{\theta} \log \pi_{\theta}(a \mid x).
\end{split}    
    \label{eq:onpolicy_grad}
\end{equation}
\citet{Williams1988} originally proposed to use the average observed reward for $\beta$. Subsequent work has derived optimal baselines for general on-policy RL scenarios~\cite{Dayan1991, Greensmith2004}.
However, to the best of our knowledge, \emph{optimal baselines for on-policy contextual bandits have not been considered in previous work}.

\noindent
\textbf{Optimal baseline for on-policy bandits.}
The optimal baseline $\beta$ for the on-policy gradient estimate in Eq.~\ref{eq:onpolicy_grad} is the one that minimizes the variance of the gradient estimate.
In accordance with earlier work~\cite{Greensmith2004},
we define the variance of a vector random variable as the sum of the variance of its individual components.
Therefore, the optimal baseline is given by:
%
\begin{align}
    & \argmin_\beta \mathrm{Var} \mleft( \widehat{\nabla_{\theta}(V_{\beta}(\pi_{\theta}))} \mright) \nonumber \\
    &= \argmin_\beta  \frac{1}{\mleft|\mathcal{D}\mright|} \mathop{\mathrm{Var}}\mleft[ \nabla_{\theta}\mleft( \log(\pi_{\theta}(a|x)) \mleft( r - \beta \mright) \mright) \mright]  \\
    &= \argmin_\beta \frac{1}{\mleft|\mathcal{D}\mright|} \mathop{\mathbb{E}}\mleft[  \nabla_{\theta} \log(\pi_{\theta}(a|x))^{\top} \nabla_{\theta} \log(\pi_{\theta}(a|x))  \mleft( r - \beta \mright)^2 \mright] \hspace{-1em} \label{eq:step2} \\
    & \quad - \frac{1}{\mleft|\mathcal{D}\mright|}  \mathop{\mathbb{E}}\mleft[  \nabla_{\theta} \log(\pi_{\theta}(a|x)) \mleft( r - \beta \mright) \ \mright]^{\top} \mathop{\mathbb{E}}\mleft[  \nabla_{\theta} \log(\pi_{\theta}(a|x)) \mleft( r - \beta \mright) \mright] \nonumber  \\
    &= \argmin_\beta  \frac{1}{\mleft|\mathcal{D}\mright|} \mathop{\mathbb{E}}\mleft[ \| \nabla_{\theta} \log(\pi_{\theta}(a|x))\|^2_2 \mleft( r -\beta \mright)^2  \mright],
    \label{eq:onpolicy_optimal_baseline}
\end{align}
where we ignore the second term in Eq.~\ref{eq:step2}, since it is independent of $\beta$~\cite[Eq. 12]{Mohamed2020}.
The result from this derivation (Eq.~\ref{eq:onpolicy_optimal_baseline}) reveals that the optimal baseline can be obtained by solving the following equation:
\begin{align}
    \mbox{}\hspace*{-2mm}
    \frac{\partial \mathrm{Var}\big(\widehat{\nabla_{\theta}(V_{\beta}(\pi_{\theta}))} \big)}{\partial \beta} &= 
    \frac{2}{\mleft|\mathcal{D}\mright|} \mathop{\mathbb{E}}\mleft[  \| \nabla_{\theta} \log(\pi_{\theta}(a \mid x))\|^2_2 \mleft( \beta -  r  \mright)  \mright] \!= 0,
    \hspace*{-1mm}\mbox{}
\end{align}
which results in the following optimal baseline correction:
\begin{equation}
    \beta^{*} =  \frac{\mathop{\mathbb{E}}\mleft[  \| \nabla_{\theta} \log(\pi_{\theta}(a|x))\|^2_2 r(a,x) \mright]}{\mathop{\mathbb{E}}\mleft[  \| \nabla_{\theta} \log(\pi_{\theta}(a|x))\|^2_2 \mright]},
\end{equation}
and the empirical estimate of the optimal baseline correction:
\begin{equation}
    \widehat{\beta^{*}} =  \frac{\sum_{(x,a,r) \in \mathcal{D}} \mleft[  \| \nabla_{\theta} \log(\pi_{\theta}(a|x))\|^2_2 r(a,x) \mright]}{\sum_{(x,a,r) \in \mathcal{D}}\mleft[  \| \nabla_{\theta} \log(\pi_{\theta}(a|x))\|^2_2 \mright]}.
\end{equation}
This derivation follows the more general derivation from \citet{Greensmith2004} for partially observable Markov decision processes (POMDPs). We have not encountered its use in the existing bandit literature applied to recommendation problems. In Section~\ref{sec:grad_var}, we show that a similar line of reasoning can be applied to derive a variance-optimal gradient for the off-policy contextual bandit setup.

\subsection{Off-policy estimation for general bandits}
\label{sec:general_OPE}

Deploying $\pi$ is a costly prerequisite for estimating $V(\pi)$, that comes with the risk of deploying a possible poorly valued $\pi$.
Therefore, commonly in real-world model validation pipelines, practitioners wish to estimate $V(\pi)$ \emph{before} deployment.
Accordingly, we will address this \emph{counterfactual} evaluation scenario that falls inside the field of off-policy estimation (OPE)~\cite{Saito2021_OPE,Vasile2020}. 

The expectation $V(\pi)$ can be unbiasedly estimated using samples from a \emph{different} policy $\pi_{0}$ through \emph{importance sampling}, also known as inverse propensity score weighting (IPS)~\cite[\S 9]{Owen2013}:
\begin{equation}\label{eq:imp_sampl}
    \mathop{\mathbb{E}}_{x\sim\mathsf{P}(X)}\mleft[\mathop{\mathbb{E}}_{a \sim \pi(\cdot|x)}\mleft[ R \mright]\mright]
    = \mathop{\mathbb{E}}_{x\sim\mathsf{P}(X)}\mleft[\mathop{\mathbb{E}}_{a \sim \pi_{0}(\cdot|x)}\mleft[ \frac{\pi(a\mid x)}{\pi_{0}(a\mid x)} R \mright]\mright].
\end{equation}
To ensure that the so-called \emph{importance weights} $\frac{\pi(a\mid x)}{\pi_{0}(a\mid x)}$ are well-defined, we assume ``\emph{common support}'' by the logging policy: $\forall a \in \mathcal{A}, x \in \mathcal{X}: \pi(a\mid x) > 0 \implies \pi_{0}(a\mid x) > 0$.

From Eq.~\ref{eq:imp_sampl}, we can derive an unbiased estimator for $V(\pi)$ using contexts, actions and rewards logged under $\pi_{0}$, denoted by $\mathcal{D}$:
\begin{equation}\label{eq:IPS}
    \widehat{V}_{\rm IPS}(\pi,\mathcal{D}) = \frac{1}{\mleft|\mathcal{D}\mright|} \sum_{(x,a,r) \in \mathcal{D}} \frac{\pi(a\mid x)}{\pi_{0}(a\mid x)} r.
\end{equation}
To keep our notation brief, we suppress subscripts when they are clear from the context.
In the context of gradient-based optimization methods, we often refer to a minibatch $\mathcal{B} \subset \mathcal{D}$ instead of the whole dataset, as is typical for, e.g., stochastic gradient descent (SGD).

If we wish to learn a policy that maximises this estimator, we need to estimate its gradient for a batch $\mathcal{B}$.
Whilst some previous work has applied a REINFORCE estimator~\cite{chen2019top,ma2020off,chen2022actorcritic}, we use a straightforward Monte Carlo estimate for the gradient:
\begin{equation}\label{eq:MC_gradient_IPS}
    \nabla \widehat{V}_{\rm IPS}(\pi, \mathcal{B}) = \frac{1}{\mleft|\mathcal{B}\mright|} \sum_{(x,a,r) \in \mathcal{B}} \frac{\nabla\pi(a|x)}{\pi_{0}(a|x)} r.
\end{equation}
Importance sampling --- the bread and butter of unbiased off-policy estimation --- often leads to increased variance compared to on-policy estimators.
Several variance reduction techniques have been proposed specifically to combat the excessive variance of $\widehat{V}_{\rm IPS}$~\cite{Ionides2008,Dudik2014,Swaminathan2015}.
Within the scope of this work, we only consider techniques that reduce variance \emph{without} introducing bias.

\noindent
\textbf{Self-normalised importance sampling.}
The key idea behind \emph{self-normalisation}~\cite[\S 9.2]{Owen2013} is to use a \emph{multiplicative} control variate to rescale $\widehat{V}_{\rm IPS}(\pi,\mathcal{D})$.
An important observation for this approach is that for any policy $\pi$ and a dataset $\mathcal{D}$ logged under $\pi_{0}$, the expected average of importance weights should equal 1~\cite[\S 5]{Swaminathan2015}:
\begin{equation}\label{eq:control_variate}
   \mathop{\mathbb{E}}_{\mathcal{D} \sim \mathsf{P}(\mathcal{D})}\mleft[
   \frac{1}{\mleft|\mathcal{D}\mright|} \sum_{(x,a,r) \in \mathcal{D}}\frac{\pi(a\mid x)}{\pi_{0}(a\mid x)}
   \mright] = 1.
\end{equation}
Furthermore, as this random variable (Eq.~\ref{eq:control_variate}) is likely to be correlated with the IPS estimates, we can expect that its use as a control variate will lead to reduced variance (see~\cite[e.g.,][]{Kong1992}).
This gives rise to the asymptotically unbiased and parameter-free self-normalised IPS (SNIPS) estimator, with $ S \coloneqq  \frac{1}{D} \sum_{(x,a,r) \in \mathcal{D}} \frac{\pi(a|x)}{\pi_{0}(a|x)}$ as its normalization term:
\begin{equation}\label{eq:SNIPS}
    \widehat{V}_{\rm SNIPS}(\pi,\mathcal{D}) =  \frac{\sum_{(x,a,r) \in \mathcal{D}} \frac{\pi(a|x)}{\pi_{0}(a|x)} r}{\sum_{(x,a,r) \in \mathcal{D}} \frac{\pi(a|x)}{\pi_{0}(a|x)}}
    =
    \frac{\widehat{V}_{\rm IPS}(\pi,\mathcal{D})}{S}.
\end{equation}
Given the properties of being asymptotically unbiased and para\-meter-free, this estimator is often a go-to method for off-policy \emph{evaluation} use-cases~\cite{Saito2021_OPE}.
An additional advantage is that the SNIPS estimator is invariant to translations in the reward, which cannot be said for $\widehat{V}_{\rm IPS}$.
Whilst the formulation in Eq.~\ref{eq:SNIPS} is not obvious in this regard, it becomes clear when we consider its gradient:
%
\begin{align}\label{eq:SNIPS_gradient}
    &\nabla \widehat{V}_{\rm SNIPS}(\pi,\mathcal{D}) =  \nabla\mleft(\frac{\sum_{(x,a,r)} \frac{\pi(a|x)}{\pi_{0}(a|x)} r}{\sum_{(x,a)} \frac{\pi(a|x)}{\pi_{0}(a|x)}}\mright)  \nonumber \\ 
 &\qquad\qquad\qquad =\frac{\mleft( \sum_{(x,a,r)} \frac{\nabla \pi(a|x)}{\pi_{0}(a|x)} r  \mright) \mleft( \sum_{(x,a)} \frac{ \pi(a|x)}{\pi_{0}(a|x)}  \mright) }{\mleft({\sum_{(x,a)} \frac{ \pi(a|x)}{\pi_{0}(a|x)}}\mright)^2}  \nonumber \\
 & \quad  - \frac{\mleft( \sum_{(x,a,r)} \frac{ \pi(a|x)}{\pi_{0}(a|x)} r  \mright) \mleft( \sum_{(x,a)} \frac{\nabla \pi(a|x)}{\pi_{0}(a|x)}  \mright) }{\mleft({\sum_{(x,a)} \frac{  \pi(a|x)}{\pi_{0}(a|x)}}\mright)^2} \\
 &= \frac{\sum_{(x_{i},a_{i}, r_{i})} \sum_{(x_{j},a_{j}, r_{j})}\frac{\pi(a_{i}|x_{i})\nabla\pi(a_{j}|x_{j})}{\pi_{0}(a_{i}|x_{i})\pi_{0}(a_{j}|x_{j})}(r_{j}-r_{i}) }{\mleft({\sum_{(x,a)} \frac{ \pi(a|x)}{\pi_{0}(a|x)}}\mright)^2}\nonumber  \\ 
 &= \frac{\sum\limits_{(x_{i},a_{i}, r_{i})} \sum\limits_{(x_{j},a_{j}, r_{j})}\frac{\pi(a_{i}|x_{i})\pi(a_{j}|x_{j})}{\pi_{0}(a_{i}|x_{i})\pi_{0}(a_{j}|x_{j})} \nabla\log\pi(a_{j}|x_{j})(r_{j}-r_{i}) }{\mleft({\sum_{(x,a)} \frac{ \pi(a|x)}{\pi_{0}(a|x)}}\mright)^2}.\nonumber
\end{align}
Indeed, as the SNIPS gradient relies on the \emph{relative difference} in observed reward between two samples, a constant correction would not affect it (i.e., if $\overline{r} = r - \beta$, then $r_j-r_i \equiv \overline{r}_{j}-\overline{r}_{i}$).

\citet{Swaminathan2015} effectively apply the SNIPS estimator (with a variance regularisation term~\cite{Swaminathan2015_BLBF}) to off-policy \emph{learning} scenarios.
Note that while $\widehat{V}_{\rm IPS}$ neatly decomposes into a single sum over samples, $\widehat{V}_{\rm SNIPS}$ no longer does.
Whilst this may be clear from the gradient formulation in Eq.~\ref{eq:SNIPS_gradient}, a formal proof can be found in~\cite[App. C]{Joachims2018}.
This implies that mini-batch optimization methods (which are often necessary to support learning from large datasets) are no longer directly applicable to $\widehat{V}_{\rm SNIPS}$.

\citet{Joachims2018} solve this by re-framing the task of maximising $\widehat{V}_{\rm SNIPS}$ as an optimization problem on $\widehat{V}_{\rm IPS}$ with a constraint on the self-normalisation term.
That is, if we define:
\begin{equation}
\mbox{}\hspace*{-2mm}
    \pi^{\star} \!=\! \argmax_{\pi \in \Pi} \widehat{V}_{\rm SNIPS}(\pi, \mathcal{D}), \text{ with } S^{\star} \!=\! \frac{1}{\mleft|\mathcal{D}\mright|}
    \!\sum_{(x,a,r) \in \mathcal{D}} \!
    \frac{\pi^{\star}(a|x)}{\pi_{0}(a|x)},
\end{equation}
then, we can equivalently state this as:
\begin{equation}\label{eq:constr_opt}
    \pi^{\star} = \argmax_{\pi \in \Pi} \widehat{V}_{\rm IPS}(\pi, \mathcal{D}), \text{ s.th. } \frac{1}{\mleft|\mathcal{D}\mright|}\sum_{(x,a,r) \in \mathcal{D}} \frac{\pi(a|x)}{\pi_{0}(a|x)} = S^{\star}.
\end{equation}
\citet{Joachims2018} show via the Lagrange multiplier method that this optimization problem can be solved by optimising for $\widehat{V}_{\rm IPS}$ with a translation on the reward:
\begin{equation}\label{eq:banditnet}
\begin{split}
    \pi^{\star} =  \argmax_{\pi \in \Pi} \widehat{V}_{\lambda^{\star}\text{-}{\rm IPS}}(\pi, \mathcal{D}), \text{where }\\ 
    \widehat{V}_{\lambda\text{-}{\rm IPS}}(\pi,\mathcal{D}) = \frac{1}{\mleft|\mathcal{D}\mright|} \sum_{(x,a,r) \in \mathcal{D}} \frac{\pi(a|x)}{\pi_{0}(a|x)} (r-\lambda).
\end{split}
\end{equation}
This approach is called BanditNet~\cite{Joachims2018}.
Naturally, we do not know $\lambda^{\star}$ beforehand (because we do not know $S^{\star}$), but we know that $S^{\star}$ should concentrate around 1 for large datasets (see Eq.~\ref{eq:control_variate}).
\citet{Joachims2018} essentially propose to treat $\lambda$ as a hyper-parameter to be tuned in order to find $S^{\star}$.

\noindent
\textbf{Doubly robust estimation. }
Another way to reduce the variance of $\widehat{V}_{\rm IPS}$ is to use a model of the reward $\widehat{r}(a,x)\approx \mathbb{E}[R|X=x;A=a]$.
Including it as an additive control variate in Eq.~\ref{eq:IPS} gives rise to the doubly robust (DR) estimator, deriving its name from its unbiasedness if \emph{either} the logging propensities $\pi_{0}$ or the reward model $\widehat{r}$ is unbiased~\cite{Dudik2014}:
\begin{align}\label{eq:DR}
    &\widehat{V}_{\rm DR}(\pi,\mathcal{D}) = \\
    &\;\;\; \frac{1}{\mleft|\mathcal{D}\mright|} \sum_{(x,a,r) \in \mathcal{D}} \mleft(\frac{\pi(a\mid x)}{\pi_{0}(a\mid x)} (r-\widehat{r}(a,x)) + \sum_{a^{\prime} \in \mathcal{A}} \pi(a^{\prime}\mid x)\widehat{r}(a^{\prime},x)\mright). \nonumber
\end{align}
Several further extensions have been proposed in the literature: one can optimize the reward model $\widehat{r}(a,x)$ to minimize the resulting variance of $\widehat{V}_{\rm DR}$~\cite{Farajtabar2018}, further parameterise the trade-off relying on $\widehat{V}_{\rm IPS}$ or  $\widehat{r}(a,x)$~\cite{Su2019}, or shrink the IPS weights to minimize a bound on the MSE of the resulting estimator~\cite{su2020doubly}.
One disadvantage of this method, is that practitioners are required to fit the secondary reward model $\widehat{r}(a,x)$, which might be costly and sample inefficient.
Furthermore, variance reduction is generally not guaranteed, and stand-alone $\widehat{V}_{\rm IPS}$ can be empirically superior in some scenarios~\cite{Jeunen2020REVEAL}.

\section{Unifying Off-Policy Estimators}
Section~\ref{sec:sec1} provides an overview of (asymptotically) unbiased estimators for the value of a policy.
We have introduced the contextual bandit setting, detailing often used variance reduction techniques for both on-policy (i.e., regression adjustments and baseline corrections) and off-policy estimation (i.e., self-normalisation and doubly robust estimation).
In this section, we demonstrate that they perform equivalent optimization as baseline-corrected estimation.
Subsequently, we characterize the baseline corrections that either minimize the variance of the estimator, or that of its gradient.

\subsection{A unified off-policy estimator}

\textbf{Baseline corrections for $\nabla \widehat{V}_{{\rm IPS}}(\pi,\mathcal{D})$.}
Baseline corrections are common in on-policy estimation, but occur less often in the off-policy literature.
The estimator is obtained by removing a baseline control variate $\beta \in \mathbb{R}$ from the reward of each action, while also adding it to the estimator:
\begin{equation}    
    \widehat{V}_{\beta\text{-}{\rm IPS}} = \beta + \frac{1}{\mleft|\mathcal{D}\mright|}\sum_{(x,a,r) \in \mathcal{D}} \frac{\pi(a|x)}{\pi_{0}(a|x)} (r-\beta).
    \label{eq:beta_ips_ope}
\end{equation} 
Its unbiasedness is easily verified:
\begin{equation}
\begin{split}
    \mathop{\mathbb{E}} \mleft[\widehat{V}_{\beta\text{-}{\rm IPS}} \mright] &= \mathop{\mathbb{E}} \mleft[\beta \mright] + \mathop{\mathbb{E}} \mleft[\frac{\pi(a|x)}{\pi_{0}(a|x)} (r-\beta)  \mright] \\ 
    &= \beta + \mathop{\mathbb{E}} \mleft[\frac{\pi(a|x)}{\pi_{0}(a|x)} r  \mright] - \beta \qquad = V(\pi).
\end{split}
\label{eq:betaIPSunbiased}
\end{equation}
From an optimization perspective, we are mainly interested in the gradient of the $\widehat{V}_{{\rm \beta\text{-}IPS}}$ objective:
\begin{equation}\label{eq:MC_gradient_bIPS}
    \nabla \widehat{V}_{\beta\text{-}{\rm IPS}}(\pi,\mathcal{B}) = \frac{1}{\mleft|\mathcal{B}\mright|} \sum_{(x,a,r) \in \mathcal{B}}  \frac{ \nabla \pi(a|x)}{\pi_{0}(a|x)} \mleft(r - \beta \mright) .
\end{equation}
Our key insight is that SNIPS and certain doubly-robust estimators have an equivalent gradient to the proposed $\beta$-IPS estimator.
As a result, optimizing them is equivalent to optimizing $\widehat{V}_{\beta\text{-}{\rm IPS}}$ for a specific $\beta$ value.

\noindent\textbf{Self-normalisation through BanditNet and $\widehat{V}_{\lambda\text{-}{\rm IPS}}(\pi,\mathcal{D})$.}
If we consider the optimization problem for SNIPS that is solved by BanditNet in Eq.~\ref{eq:banditnet}~\cite{Joachims2018}, we see that its gradient is given by:
\begin{equation}
    \label{eq:banditnet_grad}
    \nabla \widehat{V}_{\lambda\text{-}{\rm IPS}}(\pi,\mathcal{B}) = \frac{1}{\mleft|\mathcal{B}\mright|} \sum_{(x,a,r) \in \mathcal{B}}  \frac{ \nabla \pi(a|x)}{\pi_{0}(a|x)} \mleft(r - \lambda \mright).
\end{equation}

\noindent \textbf{Doubly robust estimation via $\widehat{V}_{\rm DR}(\pi,\mathcal{D})$}.
As mentioned, a nuisance of doubly robust estimators is the requirement of fitting a regression model $\widehat{r}(a,x)$.
Suppose that we instead treat $\widehat{r}$ as a single scalar hyper-parameter, akin to the BanditNet approach.
Then, the gradient of such an estimator would be given by:
\begin{equation}
    \nabla \widehat{V}_{\widehat{r}\text{-}{\rm DR}}(\pi,\mathcal{B}) = \frac{1}{\mleft|\mathcal{B}\mright|} \sum_{(x,a,r) \in \mathcal{B}}  \frac{ \nabla \pi(a|x)}{\pi_{0}(a|x)} \mleft(r - \widehat{r} \mright).
    \label{eq:dr_grad_const_rew}
\end{equation}

\noindent%
Importantly, these three approaches are motivated through entirely different lenses: minimizing gradient variance, applying a multiplicative control variate to reduce estimation variance, and applying an additive control variate to improve robustness.
But they result in equivalent gradients, and thus, in equivalent optima.
Specifically, for optimization, the estimators are equivalent when $\beta\equiv\lambda\equiv\widehat{r}$.

This equivalence implies that the choice between these three approaches is not important.
Since the simple baseline correction estimator $\widehat{V}_{\beta\text{-}{\rm IPS}}$ (Eq.~\ref{eq:beta_ips_ope}) has an equivalence with all SNIPS estimators and all doubly-robust estimators with a constant reward, we propose that $\widehat{V}_{\beta\text{-}{\rm IPS}}$ should be seen as an estimator that unifies all three approaches.
Accordingly, we argue that the real task is to find the optimal $\beta$ value for $\widehat{V}_{\beta\text{-}{\rm IPS}}$, since this results in an estimator that is at least as optimal as any estimator in the underlying families of estimators, and possibly superior to them.

The remainder of this section describes the optimal $\beta$ values for minimizing gradient variance and estimation value variance.

\subsection{Minimizing gradient variance}
\label{sec:grad_var}
Similar to the on-policy variant derived in Eq.~\ref{eq:onpolicy_optimal_baseline}, we can derive the optimal baseline in the off-policy case as the one which results in the minimum variance for the gradient estimate given by Eq.~\ref{eq:MC_gradient_IPS}:
\begin{align}
    \argmin_\beta \mathrm{Var} & \mleft( \nabla_{\theta} (\widehat{V}_{\beta\text{-}{\rm IPS}}(\pi_{\theta},\mathcal{B}))  \mright) \\
    &= \argmin_\beta  \frac{1}{\mleft|\mathcal{B}\mright|}  \mathop{\mathrm{Var}}\mleft[ \frac{ \nabla\pi(a|x)}{\pi_{0}(a|x)}  \mleft(r - \beta\mright)  \mright] \\
    &= \argmin_\beta  \frac{1}{\mleft|\mathcal{B}\mright|}  \mathop{\mathbb{E}}\mleft[  \| \nabla \pi(a|x))\|^2_2  \Big( \frac{r - \beta}{\pi_{0}(a|x)} \Big)^2 \ \mright] \label{eq:grad_var_step2} \\
    &\qquad\qquad\qquad - \frac{1}{\mleft|\mathcal{B}\mright|} \|  
 \mathop{\mathbb{E}}\mleft[  \frac{ \nabla\pi(a|x)}{\pi_{0}(a|x)}  \mleft(r - \beta\mright)  \mright]  \|^2_2 \nonumber \\
    &=\argmin_\beta  \frac{1}{\mleft|\mathcal{B}\mright|}  \mathop{\mathbb{E}}\mleft[  \frac{ \| \nabla \pi(a|x))\|^2_2}{\pi_{0}(a|x)^2}  \mleft( r -\beta \mright)^2  \ \mright] \label{eq:quadratic}, 
\end{align}
where we can ignore the second term of the variance in Eq.~\ref{eq:grad_var_step2}, since it is independent of $\beta$~\cite[Eq. 12]{Mohamed2020}. The optimal baseline can be obtained by solving for:
\begin{equation}
    \frac{\partial \mathrm{Var}\mleft(\nabla (\widehat{V}_{\beta\text{-}{\rm IPS}}(\pi,\mathcal{B}))\mright)}{\partial \beta} = 
    \frac{2}{\mleft|\mathcal{B}\mright|}  \mathop{\mathbb{E}}\mleft[ \frac{ \| \nabla \pi(a|x))\|^2_2}{\pi_{0}(a|x)^2} \mleft( \beta -  r  \mright) \mright] = 0,
\label{eq:first_order_condition}
\end{equation}
which results in the following optimal baseline:
\begin{equation}\label{eq:optimal_offpolicy}
    \beta^{*} =  \frac{\mathop{\mathbb{E}}\limits_{x, a \sim \pi_{0}, r}\mleft[ \frac{ \| \nabla \pi(a|x))\|^2_2}{\pi_{0}(a|x)^2} r(a,x) \mright]}{\mathop{\mathbb{E}}\limits_{x, a \sim \pi_{0}, r}\mleft[  \frac{ \| \nabla \pi(a|x))\|^2_2}{\pi_{0}(a|x)^2} \mright]},
\end{equation}
with its empirical estimate given by:
\begin{equation}
    \widehat{\beta^{*}} =  \frac{\sum_{(x,a,r) \in \mathcal{B}} \mleft[  \frac{ \| \nabla \pi(a|x))\|^2_2}{\pi_{0}(a|x)^2} r \mright]}{\sum_{(x,a,r) \in \mathcal{B}}\mleft[  \frac{ \| \nabla \pi(a|x))\|^2_2}{\pi_{0}(a|x)^2} \mright]}.
\end{equation}
%



\noindent%
Note that this expectation is over actions sampled by the \emph{logging} policy.
As a result, we can obtain Monte Carlo estimates of the corresponding expectations.
The derivation has high similarity with the on-policy case (cf.\ Section~\ref{sec:general})~\cite{Greensmith2004}.
Nevertheless, we are unaware of any work on off-policy learning that uses it.
\citet{Joachims2018} refer to the on-policy variant with: ``\emph{we cannot sample new roll-outs from the current policy under consideration, which means we cannot use the standard variance-optimal estimator used in REINFORCE}.''
Since the expectation is over actions sampled by the \emph{logging} policy and not the \emph{target} policy, we have shown that we do not need new roll-outs.
Thereby, our estimation strategy is a novel off-policy approach that estimates the variance-optimal baseline.

\begin{theorem}
\label{thrm:min_grad_var}
Within the family of gradient estimators with a global additive control variate, i.e., $\beta$-IPS (Eq.~\ref{eq:MC_gradient_bIPS}), IPS (Eq.~\ref{eq:MC_gradient_IPS}), BanditNet (Eq.~\ref{eq:banditnet_grad}), and DR with a constant correction (Eq.~\ref{eq:dr_grad_const_rew}), $\beta$-IPS with our proposed choice of $\beta$ in Eq.~\ref{eq:optimal_offpolicy} has minimal gradient variance.
\end{theorem}
\begin{proof}
Eq.~\ref{eq:first_order_condition} shows that the $\beta$ value in Eq.~\ref{eq:optimal_offpolicy} attains a minimum.
Because the variance of the gradient estimate (Eq.~\ref{eq:grad_var_step2}) is a quadratic function of $\beta$, and hence a convex function (Eq.~\ref{eq:quadratic}), it must be the global minimum for the gradient variance. 
\end{proof}

\subsection{Minimizing estimation variance}
\label{sec:estimator_var}

Besides minimizing gradient variance, one can also aim to minimize the variance of estimation, i.e., the variance of the estimated value.
We note that the $\beta$ value for minimizing estimation need not be the same value that minimizes gradient variance.
Furthermore, since $\widehat{V}_{\beta\text{-}{\rm IPS}}$ is unbiased, any estimation error will entirely be driven by variance.
As a result, the value for $\beta$ that results in minimal variance will also result in minimal estimation error:
\begin{align}
    & \argmin_\beta \mathrm{Var} \mleft( \widehat{V}_{\beta\text{-}{\rm IPS}}(\pi, \mathcal{D}) \mright) \\
    &\qquad = \argmin_\beta  \frac{1}{\mleft|\mathcal{D}\mright|} \mathop{\mathrm{Var}}\mleft[ \frac{ \pi(a|x)}{\pi_{0}(a|x)}  \mleft(r - \beta\mright)  \mright] \\
    &\qquad = \argmin_\beta \frac{1}{\mleft|\mathcal{D}\mright|}  \mathop{\mathbb{E}}\mleft[ \Big( \frac{ \pi(a|x)}{\pi_{0}(a|x)}  \mleft(r - \beta\mright) \Big)^2 \ \mright]\\
    & \qquad\qquad\qquad\qquad\qquad - \frac{1}{\mleft|\mathcal{D}\mright|} \Big(  
 \mathop{\mathbb{E}}\mleft[  \frac{ \pi(a|x)}{\pi_{0}(a|x)}  \mleft(r - \beta\mright)  \mright]  \Big)^2 \nonumber \\
    &\qquad = \argmin_\beta \frac{1}{\mleft|\mathcal{D}\mright|}  \mathop{\mathbb{E}}\mleft[  \mleft(\frac{ \pi(a|x)}{\pi_{0}(a|x)}\mright) ^2 \mleft( r -\beta \mright)^2  \ \mright]\\
    & \qquad\qquad\qquad\qquad\qquad  - \frac{1}{\mleft|\mathcal{D}\mright|} \Big(  
 \mathop{\mathbb{E}}\mleft[  \frac{ \pi(a|x)}{\pi_{0}(a|x)}  r  \mright] - \beta  \Big)^2.  \nonumber 
\end{align}
The minimum is obtained by solving for the following equation:
\begin{align}
    &\frac{\partial \mleft(\mathrm{Var} \mleft( \widehat{V}_{\beta\text{-}{\rm IPS}}(\pi, \mathcal{D}) \mright)\mright)}{\partial \beta} \label{eq:gradientzeroMSE}  \\
    &\quad\; = \frac{2}{\mleft|\mathcal{D}\mright|}  \mathop{\mathbb{E}}\mleft[  \mleft(\frac{\pi(a|x)}{\pi_{0}(a|x)}\mright) ^2\mleft( \beta -  r  \mright)  \ \mright] - \frac{2}{\mleft|\mathcal{D}\mright|} \mleft( \beta - \mathop{\mathbb{E}}\mleft[  \frac{ \pi(a|x)}{\pi_{0}(a|x)}  r  \mright] \mright) = 0, \nonumber
\end{align}
%
which results in the following optimal baseline:
\begin{align}\label{eq:optimal_offpolicy_beta}
    \beta^{*} &=  \frac{\mathop{\mathbb{E}}\mleft[  \mleft(\mleft(\frac{ \pi(a|x)}{\pi_{0}(a|x)}\mright) ^2  - \frac{ \pi(a|x)}{\pi_{0}(a|x)}\mright) r(a,x)  \mright]}{\mathop{\mathbb{E}}\mleft[  \mleft(\frac{ \pi(a|x)}{\pi_{0}(a|x)}\mright) ^2 - \mleft(\frac{ \pi(a|x)}{\pi_{0}(a|x)}\mright) \mright]}.
\end{align}
%
We can estimate $\beta^{*}$ using logged data, resulting in a Monte Carlo estimate of the optimal baseline.
Such a sample estimate will not be unbiased (because it is a ratio of expectations), but the bias will vanish asymptotically (similar to the bias of the $\widehat{V}_{\rm SNIPS}$ estimator). 

Next, we formally prove that optimal estimator variance leads to overall optimality (in terms of the MSE of the estimator).

\begin{theorem}
Within the family of offline estimators with a global additive control variate, i.e., $\beta$-IPS (Eq.~\ref{eq:beta_ips_ope}), IPS (Eq.~\ref{eq:IPS}), and DR with a constant correction (Eq.~\ref{eq:dr_grad_const_rew}), $\beta$-IPS with our proposed $\beta$ in Eq.~\ref{eq:optimal_offpolicy_beta} has the minimum mean squared error (MSE): 
 \begin{equation}
     {\rm MSE}(\hat{V}(\pi)) \coloneq \mathbb{E}_{\mathcal{D}} \left[ (\hat{V}(\pi, \mathcal{D}) - V(\pi))^2 \right].
     \label{eq:estimator_mse}
\end{equation}
\end{theorem}
\begin{proof}
The MSE of any off-policy estimator $\hat{V}(\pi, \mathcal{D})$ can be decomposed in terms of the bias and variance of the estimator~\cite{su2020doubly}:
 \begin{equation}
     \text{MSE}(\hat{V}(\pi)) = \text{Bias}(\hat{V}(\pi), \mathcal{D})^2 + \text{Variance}(\hat{V}(\pi), \mathcal{D}),
\end{equation}
where the bias of the estimator is defined as:
\begin{equation}
\text{Bias}(\hat{V}(\pi), \mathcal{D}) = \left| \mathbb{E}_{\mathcal{D}} \mleft[ \hat{V}(\pi, \mathcal{D}) - V(\pi)\mright]  \right|,
\end{equation}
and the variance of the estimator is defined previously (see Section~\ref{sec:estimator_var}).
Eq.~\ref{eq:betaIPSunbiased} proves that $\beta$-IPS is unbiased: $\text{Bias}(\hat{V}(\pi), \mathcal{D})=0$.
Thus, the minimum variance (Eq.~\ref{eq:gradientzeroMSE}) implies minimum MSE.
\end{proof}

%
\noindent%
We note that SNIPS is not covered by this theorem, as it is only asymptotically unbiased.
As a result, the variance reduction brought on by SNIPS might be higher than that by $\beta$-IPS, but as it introduces bias, its estimation error (MSE) is not guaranteed to be better.
Our experimental results below indicate that our method is always at least as good as SNIPS, and outperforms it in most cases, in both learning and evaluation tasks.

\begin{figure*}[th]
    \centering
    \includegraphics[width=\linewidth]{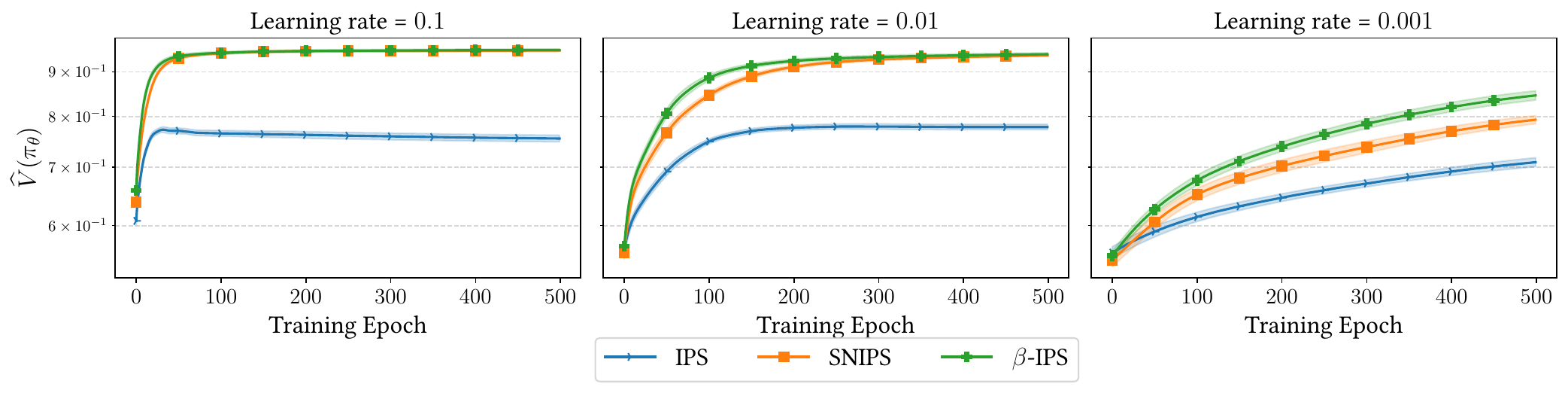}
    \caption{Performance of different off-policy learning methods trained in a full-batch gradient descent fashion in terms of the policy value on the test set. x-axis corresponds to the training epoch during the optimization (we use a maximum of 500 epochs for all methods), and y-axis corresponds to the policy value.
    A decaying learning rate is used.
    Reported results are averages over 32 independent runs with 95\% confidence interval. }
\label{fig:full_batch_val}
\end{figure*}
\begin{figure*}[th]
\vspace{-2mm}
    \centering
    \includegraphics[width=\linewidth]{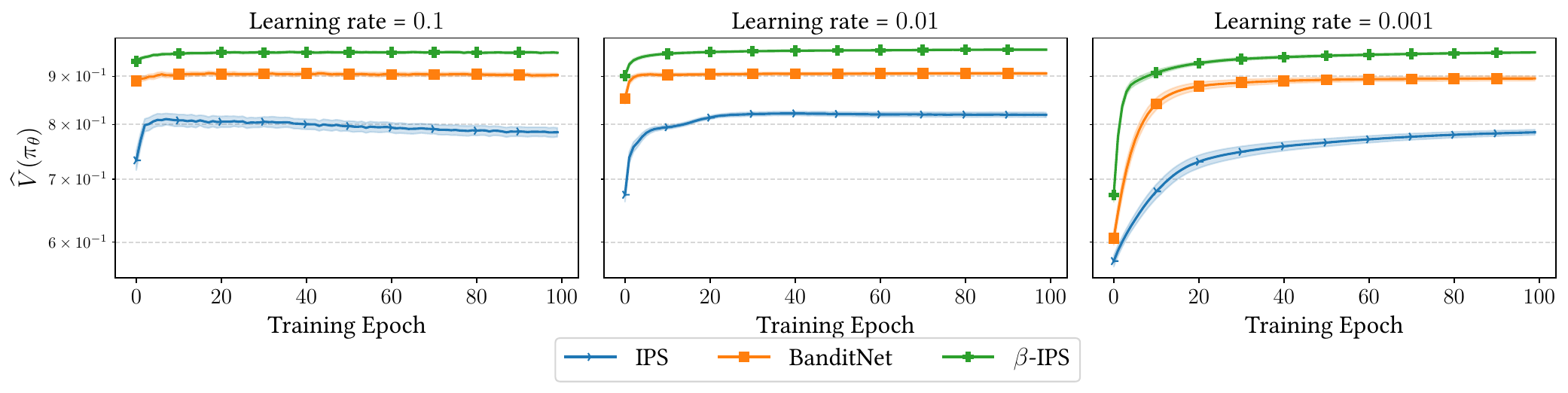}
    \caption{Performance of different off-policy learning methods trained in a mini-batch gradient descent fashion in terms of policy value on the test set. The axis labels are similar to Figure~\ref{fig:full_batch_val}.}
\label{fig:mini_batch_val}
\end{figure*}

\section{Experimental Setup}
In order to evaluate off-policy learning and evaluation methods, we need access to logged data sampled from a stochastic policy involving logging propensities (exact or estimated) along with the corresponding context and action pairs.
Recent work that focuses on off-policy learning or evaluation for contextual bandits in recommender systems follows a supervised-to-bandit conversion process to simulate a real-world bandit feedback dataset~\cite{Jeunen2021_TopK,Jeunen2021_Pessimism, Jeunen2023_AuctionGym, Saito2021_OPE,su2020doubly,Su2019}, or conducts a live experiment on actual user traffic to evaluate the policy in an \emph{on-policy} or \emph{online} fashion~\cite{chen2019top,chen2022actorcritic}.
In this work, we adopt the Open Bandit Pipeline (OBP) to simulate, in a reproducible manner, real-world recommendation setups with stochastic rewards, large action spaces, and controlled randomization~\cite{rohde2018recogym}.
Although the Open Bandit Pipeline simulates a generic offline contextual bandit setup, there is a strong correspondence to real-world recommendation setups where the environment context vector corresponds to the user context and the actions correspond to the items recommended to the user.
Finally, the reward corresponds to the user feedback received on the item (click, purchase, etc.).
As an added advantage, the simulator allows us to conduct experiments in a \emph{realistic} setting where the logging policy is sub-optimal to a controlled extent, the logged data size is limited, and the action space is large.
In addition, we conduct experiments with real-world recommendation logs from the OBP for off-policy evaluation.\footnote{\url{https://research.zozo.com/data.html}} 

The research questions we answer with our experimental results are:
\begin{enumerate}[label=\textbf{RQ\arabic*},leftmargin=*]
    \item Does the proposed \emph{estimator-variance-minimizing} baseline correction (Eq.~\ref{eq:optimal_offpolicy_beta}) improve off-policy learning (OPL) in a full-batch setting?
    \item Does the proposed \emph{gradient-variance-minimizing} baseline correction (Eq.~\ref{eq:optimal_offpolicy}) improve OPL in a mini-batch setting? 
    \item How does the proposed \emph{gradient-variance-minimizing} baseline correction (Eq.~\ref{eq:optimal_offpolicy}) affect gradient variance during OPL? 
    \item Does the proposed \emph{estimator-variance-minimizing} baseline correction (Eq.~\ref{eq:optimal_offpolicy_beta}) improve OPE performance?
\end{enumerate}
\begin{figure*}
    \centering
    \includegraphics[width=\linewidth]{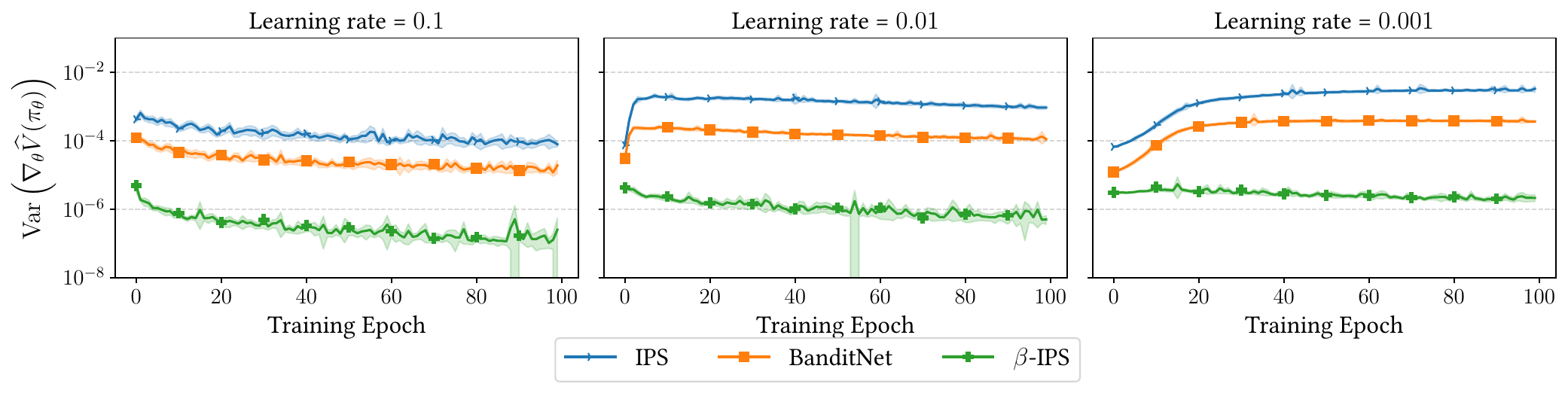}
    \caption{Empirical variance of the gradient of different off-policy learning estimators in a mini-batch optimization setup with varying learning rates (in title). We compute gradient variance for each mini-batch during training and then report the average value across all mini-batches in a training epoch. Results are averaged across 32 independent runs with 95\% confidence interval.}
\label{fig:grad_var}
\end{figure*}




\section{Results and Discussion}
\subsection{Off-policy learning performance (RQ1--3) }
To evaluate the performance of the proposed $\beta$-IPS method on an OPL task, we consider two learning setups:
\begin{enumerate}[leftmargin=*]
    \item  \emph{Full-batch}. In this setup, we  directly optimize the $\beta$-IPS policy value estimator (Eq.~\ref{eq:beta_ips_ope}) with the optimal baseline correction, which minimizes the variance of the value (Eq.~\ref{eq:optimal_offpolicy_beta}). Given that the optimal baseline correction involves a ratio of two expectations, optimizing the value function directly via a mini-batch stochastic optimization is not possible for the same reason as the SNIPS estimator, i.e., it is not possible to get an unbiased gradient estimate with a ratio function~\cite{Joachims2018}. Therefore, for this particular setting,  we use a \emph{full-batch} gradient descent method for the optimization, where the gradient is computed over the entire training dataset. 
    \item  \emph{Mini-batch}. In this setup, we focus on optimizing the $\beta$-IPS policy value estimator with the baseline correction, which minimizes the gradient estimate (Eq.~\ref{eq:optimal_offpolicy}). This setup translates to a traditional machine learning training setup where the model is optimized in a stochastic mini-batch fashion. 
\end{enumerate}

\noindent \textbf{Full-batch.} The results for the full-batch training in terms of the policy value on the test set are reported in Figure~\ref{fig:full_batch_val}, over the number of training epochs. 
To minimize the impact of external factors, we use a linear model without bias, followed by a softmax to generate a distribution over all actions, given a context vector $x$ (this is a common setup, see~\cite[e.g.,][]{Jeunen2020,Jeunen2020REVEAL,Sakhi2020}).
We note that the goal of this work is not to get the maximum possible policy value on the test set but rather to evaluate the effect of baseline corrections on gradient and estimation variance. The simple model setup allows us to easily track the empirical gradient variance, given that we have only one parameter vector. 

An advantage of the full-batch setup is that we can compute the gradient of the SNIPS estimator directly~\cite{Swaminathan2015}.
SNIPS is a natural baseline method to consider, along with the traditional IPS estimator.
Because of practical concerns, we only consider 500 epochs of optimization.
Additionally, we use the state-of-the-art and widely used Adam optimizer~\cite{kingma2014adam}. 

The IPS method converges to a lower test policy value in comparison to the SNIPS and the proposed $\beta$-IPS methods, even after 500 epochs.
A likely reason is the high-variance of the IPS estimator~\cite{Dudik2014}, which can cause it to get stuck in  bad local minima. 

The methods with a control variate, i.e., SNIPS (with multiplicative control variate) and $\beta$-IPS (with additive control variate) converge to substantially better test policy values.
In terms of the convergence speed, $\beta$-IPS converges to the optimal value faster than the SNIPS estimator, most likely because it has lower estimator variance than SNIPS. 
With this, we can answer RQ1 as follows: in the full-batch setting, our proposed optimal baseline correction enables $\beta$-IPS to converge faster than SNIPS at similar performance.
 
\noindent \textbf{Mini-batch.} The results for mini-batch training in terms of the test policy value are reported in Figure~\ref{fig:mini_batch_val}.
Different from the full-batch setup, where the focus is on reducing the variance of the \emph{estimator} value (Section~\ref{sec:estimator_var}), in the mini-batch mode, the focus is on reducing the variance of the gradient estimate (Section~\ref{sec:grad_var}).
The model and training setup are similar to the full-batch mode, except that we fixed the batch size to 1024 for the mini-batch experiments. 
Preliminary results indicated that the batch size hyper-parameter has a limited effect.

Analogous to the full-batch setup, the IPS estimator results in a lower test policy value, most likely because of the high gradient variance which prevents convergence to high performance.
In contrast, due to their baseline corrections, BanditNet (Eq.~\ref{eq:banditnet_grad}) and $\beta$-IPS have a lower gradient variance.
Accordingly, they also converge to better performance~\cite{bottou2018optimization}, i.e., resulting in superior test policy values. 

Amongst these baseline-corrected gradient-based methods (BanditNet and $\beta$-IPS), our proposed $\beta$-IPS estimator outperforms BanditNet as it provides a policy with substantially higher value.
The differences are observed over different choices of learning rates.
Thus we answer RQ2 accordingly: in the mini-batch setting, our proposed gradient-minimizing baseline method results in considerably higher policy value compared to both IPS and BanditNet.

\begin{figure*}[!ht]
    \centering
    \includegraphics[width=\linewidth]{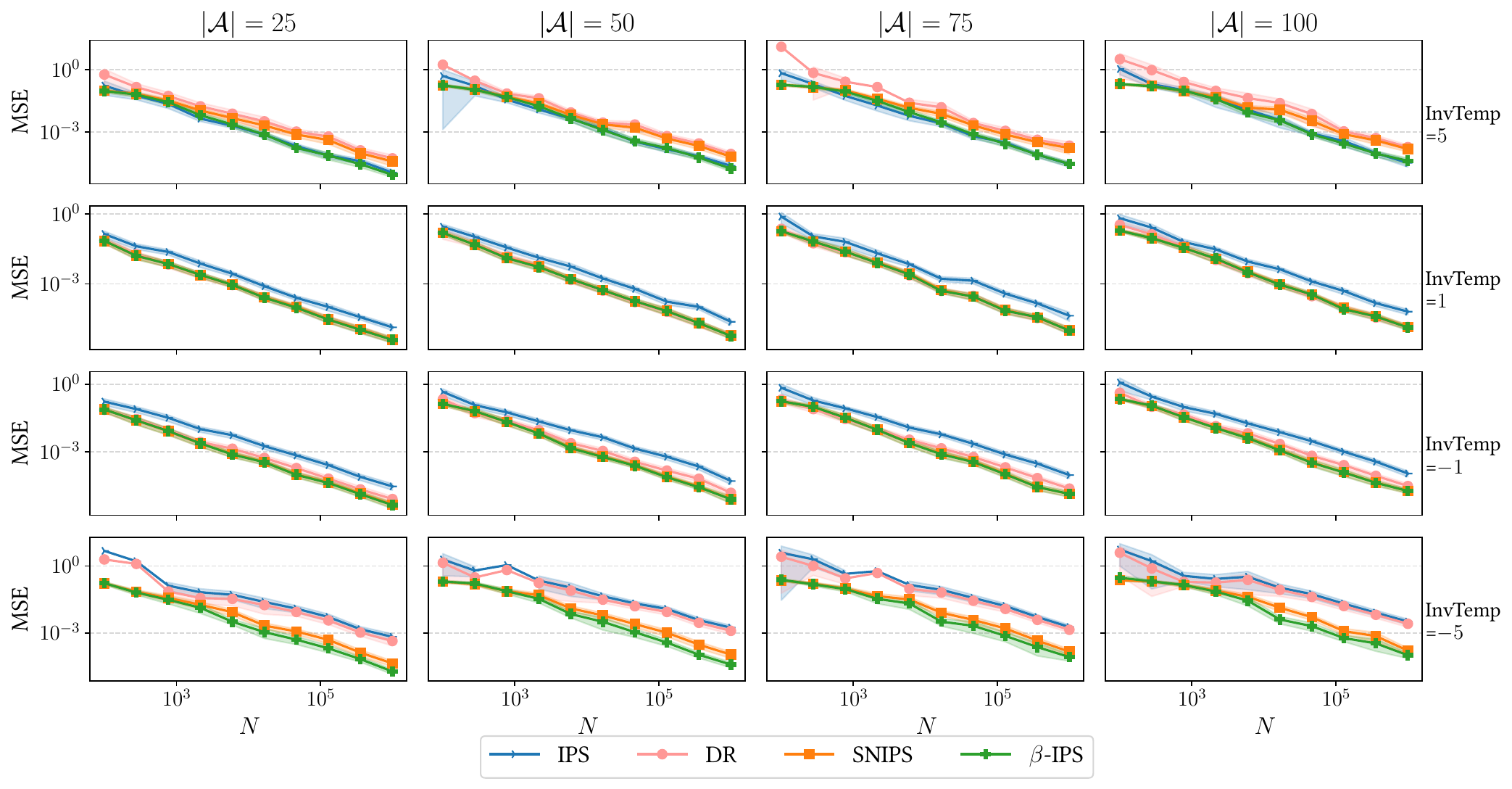}
    \caption{Mean Squared Error (MSE) of different off-policy estimators with varying action space (from left to right), and varying inverse temperature parameter of the softmax logging policy (from top to bottom). X-axis corresponds to the size of the logged data simulated (ranging from $10^2$ to $10^6$), and the y-axis corresponds to the MSE (evaluated over 100 independent samples of the synthetic data) along with 95\% confidence interval. Each row corresponds to a different setting of inverse temperature of the softmax logging policy. 
    We only consider unbiased (asymptotically or otherwise) estimators.}
    \label{fig:ope_mse}
\end{figure*}

Next, we directly consider the empirical gradient variance of different estimators;
Figure~\ref{fig:grad_var} reports the average mini-batch gradient variance per epoch.
As expected, the IPS estimator has the highest gradient variance by a large margin. 
For BanditNet, we observe a lower gradient variance, which is the desired result of the additive baseline it employs.
Finally, we observe that our proposed method $\beta$-IPS has the lowest gradient variance.
This result corroborates the theoretical claim (Theorem~\ref{thrm:min_grad_var}) that the $\beta$-IPS estimator has the lowest gradient variance amongst all global additive control variates (including IPS and BanditNet).
Our answer to RQ3 is thus clear: our proposed $\beta$-IPS results in considerably lower gradient variance compared to BanditNet and IPS.

\subsection{Off-policy evaluation performance (RQ4)}
To evaluate the performance of the proposed $\beta$-IPS method, which minimizes the estimated policy value (Eq.~\ref{eq:optimal_offpolicy_beta}), in an OPE task, results are presented in Figure~\ref{fig:ope_mse}. The target policy (to be evaluated) is a logistic regression model trained via the IPS objective on logged data and evaluated on a separate full-information test set. 
We evaluate the MSE of the estimated policy value against the \textit{true} policy value (Eq.~\ref{eq:estimator_mse}). 
To evaluate the MSE of different estimators realistically, we report results with varying degrees of the optimality of the behavior policy (decided by the inverse temperature parameter of the softmax) and with a varying cardinality of the action space. A positive (and higher) inverse softmax temperature results in a increasingly optimal behavior policy (selects action with highest reward probability), and a negative (and lower) inverse softmax temperature parameter results in an increasingly sub-optimal behavior policy (selects actions with lowest reward probability). 
Our proposed $\beta$-IPS method has the lowest MSE in all simulated settings.
Interestingly, the proposed $\beta$-IPS has a lower MSE than the DR method, which has a regression model-based control variate, arguably more powerful than the constant control variate from the proposed $\beta$-IPS method.
Similar observations have been made in previous work, e.g., \citet{Jeunen2020REVEAL} reported that the DR estimator's performance heavily depends on the logging policy.

\begin{table}[t]
\caption{Comparison of different OPE methods on real-world recommender system logs of ZOZOTOWN from a campaign targeted towards men with a uniformly random production policy. We report the mean relative absolute error (with std).}
\label{tab:method_comparison}
\centering
\begin{tabular}{lc}
\toprule
\textbf{OPE estimator} & \textbf{Abs. relative error} $\downarrow$ \\
\midrule
IPS & 0.1277 (0.0142) \\
SNIPS & 0.1113 (0.0372) \\
DR & 0.1144 (0.0366) \\
$\beta$-IPS & \textbf{0.1078 (0.0383)} \\
\bottomrule
\end{tabular}
\end{table}

Depending on the setting, we see that $\beta$-IPS either has performance comparable to the SNIPS estimator, i.e., when inverse temperature $\in \{-1, 1\}$;
or noticeably higher performance than SNIPS, i.e., when inverse temperature $\in \{-5, 5\}$.

\noindent \textbf{Real-world evaluation.} To evaluate different estimators in a real-world recommender systems setup, we report the results of OPE from the production logs of a real-world recommender system in Table~\ref{tab:method_comparison}.
Similar to the simulation setup, the proposed $\beta$-IPS has the lowest absoluate relative error amongst all estimators in the comparison.
In conclusion, we answer RQ4:
our proposed policy-value variance minimizing baseline method results in substantially improved MSE, compared to IPS, SNIPS and DR, in offline evaluation tasks that are typical recommender system use-cases.


\begin{figure*}[!ht]
    \centering
    \includegraphics[width=\textwidth]{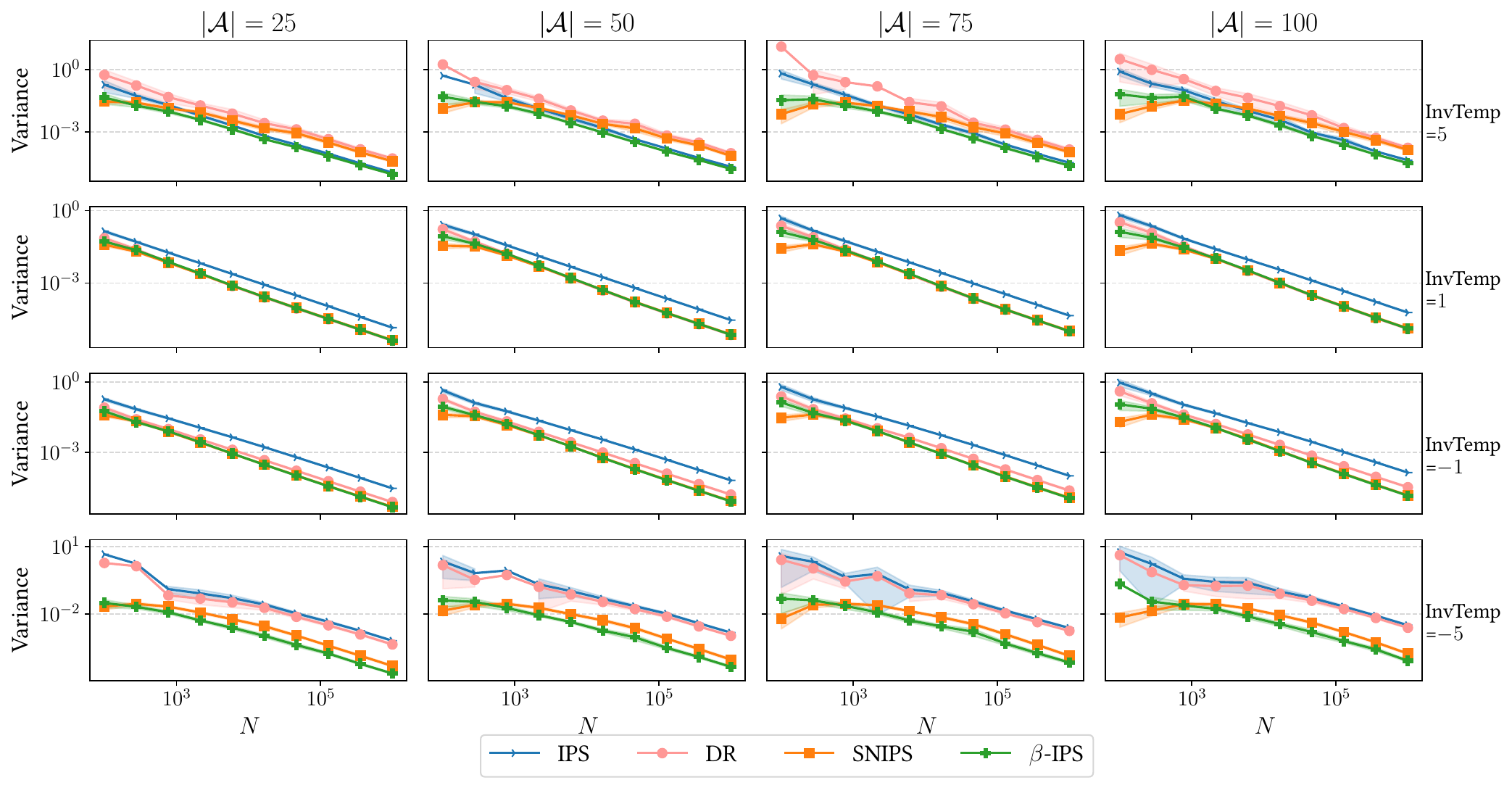}
    \caption{Empirical variance of different off-policy estimators with varying action space (from left to right), and varying sub-optimality of a temperature-based softmax behavior policy (from top to bottom). The x-axis corresponds to the size of the logged data simulated (ranging from $10^2$ to $10^6$), and the y-axis corresponds to the variance of different estimators (evaluated over 100 independent samples of the synthetic data) along with 95\% confidence interval. Each row corresponds to a different optimality level of the logging policy, decided by the inverse temperature parameter. We only consider unbiased (asymptotically or otherwise) estimators.}
    \label{fig:ope_sample_var}
\end{figure*}
\section{Conclusion and Future Work}

In this work, we have proposed to unify different off-policy estimators as equivalent additive baseline corrections.
We look at off-policy evaluation and learning settings and propose baseline corrections that minimize the variance in the estimated policy value and the empirical gradient of the off-policy learning objective.
Extensive experimental comparisons on a synthetic benchmark with realistic settings show that our proposed methods improve performance in the off-policy estimation (OPE) and off-policy learning (OPL) tasks.

We believe our work represents a significant step forward in the understanding and use of off-policy estimation methods (for both evaluation and learning use-cases), since we show that the prevalent SNIPS estimator can be improved upon with essentially no cost, as our proposed method is parameter-free and --- in contrast with SNIPS --- it retains the unbiasedness that comes with IPS.

Future work may apply a similar approach to offline reinforcement learning setups~\cite{levine2020offline}, or consider extensions of our approach for ranking applications~\cite{London2023}.


\appendix

\appendix
\section{Appendix: Off-policy Estimator Variance}

In this appendix, we report additional results from the experimental section (Section 4 from the main paper), answering RQ4. Specifically, we look the the empirical variance of various offline estimators for the task of off-policy evaluation. The mean squared error (MSE) of different offline estimators are reported in Figure~\ref{fig:ope_mse}. In this appendix, we report the empirical variance of various offline estimators in Figure~\ref{fig:ope_sample_var}. 

From the figure, it is clear that our proposed $\beta$-IPS estimator with estimator variance minimizing $\beta$ value (Eq.~\ref{eq:optimal_offpolicy_beta}) results in the lowest empirical variance in most of the cases. It is interesting to note that when the logged data is limited ($N < 10^3$), sometimes the SNIPS estimator has lower estimator variance. We suspect that the reason could be a bias in the estimate of the variance-optimal $\beta$ estimate (Eq.~\ref{eq:optimal_offpolicy_beta}), when the dataset size is small, given that it is a ratio estimate of expectations. For practical settings, i.e., when $N > 10^3$, the proposed estimator $\beta$-IPS results in a minimum sample variance, thereby empirically validating the effectiveness of our proposed $\beta$-IPS estimator for the task of OPE. 

\section*{Acknowledgements}

The authors would like to thank Adith Swaminathan, Thorsten Joachims, Ben London, and Philipp Hager for valuable discussions and feedback on early drafts of this manuscript.

This research was supported by Huawei Finland, 
the Dutch Research Council (NWO), under project numbers VI.Veni.222.269, 024.004.022, NWA.1389.20.\-183, and KICH3.LTP.20.006, 
and the European Union's Horizon Europe program under grant agreement No 101070212.
All content represents the opinion of the authors, which is not necessarily shared or endorsed by their respective employers and/or sponsors.

\clearpage
\balance
\bibliographystyle{ACM-Reference-Format}
\bibliography{bibliography}

\clearpage


\end{document}